\newcommand{\arxiv}[1]{#1}

\documentclass{article}
\usepackage{ijcai23}

\usepackage{times}
\usepackage{xspace}
\usepackage{soul}
\usepackage{url}
\usepackage[hidelinks]{hyperref}
\usepackage[utf8]{inputenc}
\usepackage[small]{caption}
\usepackage{graphicx}
\usepackage{amsmath}
\usepackage{amsthm}
\usepackage{booktabs}
\usepackage{algorithm}
\usepackage{algorithmic}
\usepackage[switch]{lineno}

\usepackage{amsxtra, amsfonts, amssymb, amstext, mathtools}
\usepackage{tikz}
\usetikzlibrary{arrows}
\usepackage{xcolor}
\usepackage{tikz}
\usepackage{pgfplots}
\pgfplotsset{compat=1.15}
\usepackage{booktabs}
\usepackage{nicefrac}
\usepackage{xspace}
\usepackage{url}
\usepackage{graphics,color}
\usepackage[algo2e,ruled,vlined,linesnumbered]{algorithm2e}
    \SetKwInOut{Input}{Input}
    \SetKwInOut{Output}{Output}
    \ResetInOut{output}
    \SetKw{Forever}{forever}
\usepackage{wrapfig}

\newtheorem{theorem}{Theorem}
\newtheorem{lemma}[theorem]{Lemma}
\newtheorem{corollary}[theorem]{Corollary}
\newtheorem{definition}[theorem]{Definition}

\newcommand{\oea}{\mbox{${(1 + 1)}$~EA}\xspace}

\newcommand{\oplea}{\mbox{${(1+\lambda)}$~EA}\xspace}

\newcommand{\mplea}{\mbox{${(\mu+\lambda)}$~EA}\xspace}

\newcommand{\opllga}{\mbox{${(1+(\lambda,\lambda))}$~GA}\xspace}
\newcommand{\ollga}{\opllga}
\newcommand{\NSGA}{\mbox{NSGA}\nobreakdash-II\xspace}

\newcommand{\onemax}{\textsc{OneMax}\xspace}
\newcommand{\LO}{\textsc{Leading\-Ones}\xspace}
\newcommand{\leadingones}{\LO}

\newcommand{\jump}{\textsc{Jump}\xspace}
\newcommand{\oneminmax}{\textsc{OneMinMax}\xspace}

\newcommand{\N}{\ensuremath{\mathbb{N}}} 

\definecolor{cqcqcq}{rgb}{0.7529411764705882,0.7529411764705882,0.7529411764705882}
\definecolor{aqaqaq}{rgb}{0.6274509803921569,0.6274509803921569,0.6274509803921569}
\definecolor{ffqqqq}{rgb}{1,0,0}
\definecolor{rvwvcq}{rgb}{0.08235294117647059,0.396078431372549,0.7529411764705882}

\newcommand{\wmin}{w_{\min}}
\newcommand{\wmax}{w_{\max}}
\newcommand{\convF}{conv($F$) }
\DeclareMathOperator{\ub}{ub}
\DeclareMathOperator{\pareto}{pareto}
\newcommand{\Spareto}{S_{\pareto}}

\let\originalleft\left
\let\originalright\right
\renewcommand{\left}{\mathopen{}\mathclose\bgroup\originalleft}
\renewcommand{\right}{\aftergroup\egroup\originalright}

\usepackage{latexsym}

\title{The First Proven Performance Guarantees for the \\
Non-Dominated Sorting Genetic Algorithm II (\NSGA) \\
on a Combinatorial Optimization Problem\arxiv{\thanks{Author-generated version of a paper appearing in the proceedings of IJCAI 2023.}}}

 \author{
 Sacha Cerf\/$^1$\and
 Benjamin Doerr$^{2}$\and
 Benjamin Hebras$^1$\and \\
 Yakob Kahane$^1$\and
 Simon Wietheger$^3$
 \affiliations
$^1$\'Ecole Polytechnique, Institut Polytechnique de Paris, Palaiseau, France\\
$^2$Laboratoire d'Informatique (LIX), CNRS, \'Ecole Polytechnique, Institut Polytechnique de Paris, Palaiseau, France\\
$^3$Hasso Plattner Institute, University of Potsdam, Germany\footnote{Work done while visiting \'Ecole Poytechnique, France.}
 \emails 
 sacha.cerf@polytechnique.edu, 
 benjamin.doerr@polytechnique.edu, 
 benjamin.hebras@polytechnique.edu, 
 yakob.kahane@polytechnique.edu, 
 simon.wietheger@student.hpi.de
}

\begin{document}

\maketitle

\begin{abstract}
 The Non-dominated Sorting Genetic Algorithm-II (NSGA-II) is one of the most prominent algorithms to solve multi-objective optimization problems. Recently, the first mathematical runtime guarantees have been obtained for this algorithm, however only for synthetic benchmark problems. 
 
 In this work, we give the first proven performance guarantees for a classic optimization problem, the NP-complete bi-objective minimum spanning tree problem. More specifically, we show that the NSGA-II with population size $N \ge 4((n-1) w_{\max} + 1)$ computes all extremal points of the Pareto front in an expected number of $O(m^2 n w_{\max} \log(n w_{\max}))$ iterations, where $n$ is the number of vertices, $m$ the number of edges, and $w_{\max}$ is the maximum edge weight in the problem instance. This result confirms, via mathematical means, the good performance of the NSGA-II observed empirically. It also shows that mathematical analyses of this algorithm are not only possible for synthetic benchmark problems, but also for more complex combinatorial optimization problems. 
  
  As a side result, we also obtain a new analysis of the performance of the  global SEMO algorithm on the bi-objective minimum spanning tree problem, which improves the previous best result by a factor of $|F|$, the number of extremal points of the Pareto front, a set that can be as large as $n w_{\max}$. The main reason for this improvement is our observation that both multi-objective evolutionary algorithms find the different extremal points in parallel rather than sequentially, as assumed in the previous proofs.
\end{abstract}

\section{Introduction}

Many optimization problems consist of several conflicting objectives. In such a situation, it is not possible to compute a single optimal solution. The most common solution concept therefore is to compute a set of Pareto optima (solutions which cannot be improved in one objective without accepting a worsening in another objective) and let a decision maker select the final solution based on their preference. 

Besides mathematical programming approaches, evolutionary algorithms (EAs) are the standard approach to multi-objective problems with many successful applications~\cite{ZhouQLZSZ11}. EAs profit here from their general ability to work with sets of solutions (``populations''). The by far most prominent multi-objective evolutionary algorithm (MOEA) is the \emph{Non-Dominated Sorting Genetic Algorithm~II} (\NSGA) proposed by Deb, Pratap, Agarwal, and Meyarivan~\cite{DebPAM02} (with over 50,000 citations on Google Scholar). 

While very successful in practice, this algorithm is only little understood from a fundamental perspective, giving the users little general advice on how to optimally employ this algorithm, e.g., how to set its parameters right. In fact, it was only at AAAI 2022 that the first mathematical runtime analysis of the \NSGA was presented~\cite{ZhengLD22}, a work that was quickly followed up more runtime analyses on this algorithm (see the previous works section). All these works analyze the performance of the \NSGA on simple benchmark problems, mostly multi-objective variants of the \onemax, \leadingones, and \jump benchmarks well-studied in the theory of single-objective randomized search heuristics~\cite{NeumannW10,AugerD11,Jansen13,ZhouYQ19,DoerrN20}. 

In this work, we conduct the first mathematical runtime analysis of the \NSGA on a classic combinatorial problem, namely the NP-complete bi-objective minimum spanning tree problem. In this problem, we are given an undirected graph with $n$ vertices and $m$ edges. In the basic single-objective version of the minimum spanning tree (MST) problem, we are also given non-negative integral edge weights and the task is to compute a minimum spanning tree. This problem is easily solved by classic algorithms. It has also been used to understand how EAs solve combinatorial optimization problems. As a first result in this direction, Neumann and Wegener~\cite{NeumannW07} showed that the \oea computes a minimum spanning tree in an expected number of $O(m^2 \log(nw_{\max}))$ iterations (and fitness evaluations). Here $w_{\max}$ denotes the maximum edge weight. Using a balanced mutation operator, this can be improved to $O(mn \log(nw_{\max}))$.

In the bi-objective variant of the problem, we are given two weight functions and the target is to compute the Pareto front of the problem of computing a spanning tree minimizing both weight functions. This problem is NP-complete, but it is possible to compute in polynomial time the extremal points of the Pareto front~\cite{HamacherR94}. The first result on how MOEAs solve this problem is \cite{Neumann07}. It shows that the  global SEMO algorithm, a multi-objective analogue of the basic \oea, computes the extremal points of the Pareto front in an expected number of $O(m^2nw_{\min}(|F|+\log(n\wmax)))$ iterations, where $F$ denotes the set of extremal points of the Pareto front and $w_{\min}$ denotes the minimum of the maximum edge weights of the two weight functions. As in \cite{NeumannW07}, this guarantee improves by a factor of $\Omega(m/n)$ when using balanced mutation.  

We note that \cite{NeumannW22} propose the convex global SEMO algorithm and show that it can solve the bi-objective MST problem in polynomial time regardless of $w_{\max}$. Given that this algorithm is very new and not yet established, we do not follow this line of research. \cite[Theorem~4]{NeumannW22} also implies a runtime bound for the classic global SEMO algorithm, but this becomes superior to ours only when the size of the Pareto front is at most $nw_{\min} / \ell$, where $\ell$ is a problem parameter which in general can only be estimated by $\Theta(m^2)$. So this result appears to be an improvement only in special cases. 

In this work, we conduct a mathematical runtime analysis of the \NSGA on the bi-objective MST problem. This is the first runtime analysis of this algorithm on a combinatorial problem. Besides showing that such analyses are possible, it proves that also the \NSGA can efficiently compute the extremal points of the Pareto front, and this in an expected number of $O(m^2 \log(n w_{\max}))$ iterations, hence a number of $O(N m^2 \log(n w_{\max}))$ fitness evaluations, when the population size $N$ is at least $N \ge 4n w_{\min}$. This results holds for various ways of generating the offspring population including the use of crossover. As in the previous works, we obtain a bound lower by a factor of $\Omega(m/n)$ when using balanced mutation.  

We note that our runtime guarantees are smaller than the ones proven in~\cite{Neumann07} for the global SEMO by essentially a factor of $|F|$, the number of extremal points, for which the only general upper bound is $n w_{\min}$. This improvement stems from our observation that the \NSGA makes progress towards the different extremal points in parallel, whereas the proof in~\cite{Neumann07} assumed that these were found sequentially. We show the same improvement for the GSEMO, lowering Neumann's bound to $O(m^2nw_{\min} \log(n w_{\max}))$.

\section{Previous Work}\label{sec:previous}

The multi-objective version of the minimum spanning tree problem, usually called multi-criteria minimum spanning tree problem (mc-MST), is an important combinatorial optimization problem with many applications in network design. We refer to Ehrgott~\cite{Ehrgott05} for an extensive discussion of the problem and the different algorithmic approaches to it. Being NP-complete, many heuristic approaches have been developed~\cite{ArroyoVV08}, including many based on evolutionary algorithms~\cite{KnowlesC00,KnowlesC01,BossekG17,ParragaDI17,MajumderKKP20}.

In this work, we investigate how the \NSGA solves the bi-objective MST problem. As in most theoretical works on randomized search heuristics, our aim is not so much finding the best possible algorithm to solve this problem (for this problem-specific algorithms will usually be superior), but we aim at understanding how a certain algorithm, here the \NSGA, solves a certain problem. The broader aim is to understand which algorithms are suitable for which problems, what are the right parameter settings, and to detect possible short-comings and remedies for these. We refer to the  textbooks~\cite{AugerD11} for a broader introduction into the research field of mathematical analyses of randomized search heuristics and its achievements.

We brief{}ly review the most relevant literature. The mathematical runtime analysis of EAs was started in the 1990s with analysis how very simple EAs such as the \oea optimize very simple benchmark problems such as \onemax or \leadingones~\cite{Muhlenbein92,Back93,Rudolph97,DrosteJW02}. The first runtime analyses of MOEAs followed a similar approach, estimating the runtime of multi-objective analogues of the \oea such as the SEMO or GSEMO on bi-objective analogues of \onemax and \leadingones~\cite{LaumannsTDZ02,Giel03,Thierens03}. 
The analyses of single-objective EAs quickly progressed towards more complex problems such as shortest paths, maximum matchings, the partition problem, the MST problem, and many others~\cite{NeumannW10}, 
or more complex algorithms such as the \oplea~\cite{JansenJW05}, \mplea~\cite{Witt06}, \mplea~\cite{AntipovD21algo}, \ollga~\cite{DoerrDE15}, 
and non-elitist algorithms~\cite{DangEL21aaai}.

In contrast, due to the more difficult population dynamics of MOEAs, the progress was slower in multi-objective evolutionary computation (there is, however, a highly successful line of research on multi-objectivization, that is, solving a single-objective problem via MOEAs~\cite{NeumannW06emo,FriedrichHHNW10,
QianYTYZ19,Crawford19}; this line of work is not very related to the research conducted in this work, though). 
There are only few mathematical results on simple MOEAs solving  combinatorial optimization problems (to the best of our knowledge only for the MST~\cite{NeumannW07,NeumannW22}, shortest path problems~\cite{Horoba09,NeumannT10}, and the travelling salesman problem~\cite{LaiZ20}). Also, only relatively few results analyze more complex algorithms such as the $(\mu+1)$ SIBEA~\cite{BrockhoffFN08,NguyenSN15,DoerrGN16}, the MOEA/D~\cite{LiZZZ16,HuangZCH19,HuangZ20}, the $(1+(\lambda,\lambda))$ GSEMO~\cite{DoerrHP22}, the \NSGA \cite{ZhengLD22,ZhengD22gecco,ZhengD22arxivmany,BianQ22,DoerrQ23tec,DoerrQ23LB,DoerrQ23crossover,DangOSS23aaai,DangOSS23gecco}, the NSGA-III \cite{WiethegerD23}, and the SMS-EMOA \cite{BianZLQ23}. Very roughly speaking, the works on the \NSGA show that this algorithm can find the Pareto front of simple bi-objective benchmark problems when the population size is chosen large enough (typically by a constant factor larger than the Pareto front). When the population size is only equal to the size of the Pareto front or when the number of objective is more than two, already on the simple \oneminmax benchmark the \NSGA needs exponential time to find the full Pareto front. So far, no mathematical runtime analysis of the \NSGA on a combinatorial optimization problem exists. 

\section{Preliminaries: Basic Notation}\label{sec:prelims}
The Multi-Objective Minimum Spanning Tree problem is stated as follows. Given an input connected graph $G = (V, E)$, and a weight function $w \colon E \longrightarrow \mathbb{N}^d$ on the edges of $G$, define the weight of any subgraph $H$ of $G$ denoted $w(H) \in \mathbb{N}^d$ as the sum of the weights of all edges present in $H$. We want to find all possible "optimum" subtree weight values $w \in \mathbb{N}^d$, in the sense that no subtree of $G$ has a weight value for which all coordinates are smaller that those of $w$, and there exists a spanning tree that has weight $w$.
Here, we focus on the case where $d = 2$.
The search space is the set of all subgraphs and is represented with $S = \{0, 1\}^m$ as a subgraph is a choice of edges.
Let $w = (w_1,w_2) \colon E \rightarrow \N^2$ be the weight function.
For a search point $s$, we refer to its weight as the geometric point of $\N^2$ denoted by $p_s = w(s)$.
We further define 
\begin{itemize}
    \item $w_{i}^{\max} = \max \{w_i(e), e \in E\}$, for $i \in \{1, 2\}$, 
    \item $w_{\max} = \max_{i\in \{1, 2\}} w_i^{\max}$
    \item $w_{\min} = \min_{i \in \{1, 2\}} w_i^{\max}$
    \item $w_{\ub} = n^2w_{\max}$.
\end{itemize} 

As in \cite{Neumann07}, the fitness of an individual $s \in S$ is given by a vector $f(s) = (f_1 (s), f_2(s))$ with 
\[f_i(s) = (c(s)-1)w_{\ub}^2  + (e(s)-(n-1))w_{\ub} + \sum_{j|s_j=1} w_i(j) \]
for $i\in \{1,2\}$, and where $w_i(j)$ is the weight of edge $e_j$ with respect to the function $w_i$. $c(s)$ is the number of connected components in the graph described by $s$, and $e(s)$ is the number of edges in this same graph. Note that for a spanning tree $s$, $f_i(s) =  \sum_{j|s_j=1} w_i(j)$.

\begin{definition}[Domination]
    For $s,s' \in S$, we say that $s$ dominates $s'$ and we note $s \preceq s'$ if $s$ Pareto-dominates $s'$ according to the fitness functions $f_1$ and $f_2$, i.e. if $f_1(s) \leq f_1(s')$ and $f_2(s) \leq f_2(s')$. We say that $s$ strictly dominates $s'$ and we note $s \prec s'$ if $s \preceq s'$ and $s$ and $s'$ have different fitness values. 
\end{definition}

Since the relation of domination only depends on the objective value, we also use "domination" to compare objective values rather than individuals.

\begin{definition}[Pareto optimality]
$s \in S$ is called \emph{Pareto optimal} if there is no search
point $s' \in S$ that strictly dominates $s$. The set of all
Pareto optimal search points $\Spareto$ is called the
Pareto set. $F = f(\Spareto)$ is the set of all Pareto optimal
objective vectors and is called the Pareto front.
\end{definition}

The goal is to find for each $q \in F$ of the considered objective function $f$ an object $s \in \Spareto$ with
$f(s) = q$. 
From now on, as in \cite{Neumann07}, we denote by \convF the lower-left part of the convex hull of $F$. The reader may find an illustration of \convF in appendix of the Arvix version of the paper.

\begin{definition}[Extremal points]
The \emph{extremal points} of the Pareto front $F$ are the vertices of the polygonal line forming \convF.
\end{definition}

Note that for each spanning tree $T$ on the convex hull there is a $\lambda \in [0, 1]$ such that T is a minimum spanning tree with respect to the single weight function $\lambda w_1 +
(1 - \lambda)w_2$ (see e.g. \cite{KnowlesC01}, \cite{Neumann07}).

Let $q_1, q_2, \ldots, q_r$ be the extremal points sorted in increasing $f_1$ value. Observe that $q_1$ (resp. $q_r$) realises by construction the minimum of $w_1$ (resp. $w_2$) in $w(S)$.

Those points are interesting because they give a solution which is a 2-approximation of the Pareto front (\cite{Neumann07}).

\subsection{Algorithms}
Here, we describe the \NSGA algorithm. Let $N$ be the population size and $n$ the chromosome size. We also give pseudocode for the global SEMO (GSEMO), since part of our results directly come from the study of this algorithm in \cite{NeumannW07}, and since most of these results also apply to this algorithm. 
\subsubsection{The \NSGA} \label{subsubseb:nsga2}

The \NSGA is an evolutionary algorithm, which means it maintains a population of $N \in \N$ solutions to an optimization problem, named \emph{individuals} (here, subgraphs of $G$). This population evolves over multiple generations. In each iteration, the algorithm generates an offspring population from selected parents in the current population, using some reproduction and mutation mechanisms. Note however that these mechanisms are not intrinsic to the \NSGA and have to be specified for the task at hand.
From the combined parent and offspring population, the algorithm decides which individuals to keep according to their fitness. 
The \NSGA is centered around two notions. 
First, \emph{rank}, which defines how good an individual is in the current population, and second, \emph{crowding distance}, which quantifies how much diversity it brings to the population.
\begin{definition}[Rank]
    Let $X$ be a population. We recursively define the rank of every individual in $X$. An individual that is not strictly dominated by any individual in $X$ has rank $1$. Then, if ranks $1, \ldots, k$ are defined, an individual that is not strictly dominated by any individual of $X$, except those of rank $1, \ldots, k$, and that is dominated by at least one individual of rank $k$, has rank $k + 1$.
\end{definition}

The ranks of individuals in a given population  $X$ can be computed in quadratic time using the fast-non-dominated-sort algorithm. The reader will find in appendix of the Arvix version of the paper the associated pseudocode.

We now introduce the notion of crowding distance.

\begin{definition}[Crowding distance]
    Let $X$ be a finite set, and $f \colon X \to \mathbb{R}$. Let $x_1, \ldots, x_N$ be the elements of $X$, sorted in increasing order of $f$ values.
    Then, the crowding distance of each $x_i$ is defined as 
    \begin{equation}
  cDis(x_i) =
    \begin{cases}
      +\infty & \text{if $i = 1$ or $i = n$}\\
      \frac{f(x_{i+1}) - f(x_i)}{f(x_l) - f(x_1)}  & \text{otherwise.}\\
    \end{cases}
\end{equation}
\end{definition}

Note that in our case, we are dealing with multiple fitness functions ($f_1,f_2$). The crowding distance of an element is then the sum of the crowding distances of this element for each function. Computing the crowding distance of each individual in a population of size $N$ can be done naively with two sortings and $O(N)$ subtractions and divisions. It can then be done in $O(N\log(N))$, which is negligible compared to fast-non-dominated sort.

The pseudocode for \NSGA can be found in Algorithm~\ref{alg:nsgaII}.

\begin{algorithm2e}[t]
 \caption{The NSGA-II.}
 \label{alg:nsgaII}
Generate initial population $P \in (\{0, 1\}^m)^N$\\
\Repeat{\Forever}{
    Generate offspring population $Q \in (\{0, 1\}^m)^N$\\
    Let $R = P \cup Q$\\
    Sort $R$ with fast-non-dominated-sort to get the sets $F_i, i \in \N$ of the individuals of rank $i$\\
    Find $i_{cut} = \max \{i \mid \sum_{k=0}^{i-1}|F_k| < N\}$\\ 
    Calculate crowding distance of each individual in $F_{i_{cut}}$\\
    Let $ \widetilde{F_{i_{cut}}}$ be the $N - \sum_{k=0}^{i_{cut}-1}|F_k|$ individuals in $F_{i_{cut}}$ with largest crowding distance, chosen at random in case of a tie\\
    $P = (\bigcup_{i=0}^{i_{cut} - 1}F_i) \cup \widetilde{F_{i_{cut}}}$
}
\end{algorithm2e}

\subsubsection{The Global Simple Evolutionary Multiobjective Optimizer (GSEMO)} GSEMO is also an evolutionary algorithm, studied on the mc-MST problem in \cite{Neumann07} and \cite{NeumannW07}. Its functioning is simpler than that of the \NSGA. 
It generates one individual $s$ at each generation, adding it to the population if it is not dominated by any other individual, and removes those that are dominated by $s$. A pseudocode for GSEMO, as described in \cite{Neumann07} can be found in Algorithm~\ref{alg:gsemo}.

\begin{algorithm2e}[t]
\caption{GSEMO}
\label{alg:gsemo}
Generate initial population $P$, which consists of a unique individual $s \in \{0, 1\}^n$, chosen randomly.\\
\Repeat{\Forever}{
    Choose a random $s \in P$\\
    Generate an offspring $s'$ from x, flipping each bit with probability $\frac{1}{n}$.\\
    \If{no individual in $P$ dominates $s'$}{
        Add $s'$ to $P$\\
        Remove all individuals in $P$ that $s'$ dominates
    }
}
\end{algorithm2e}

\section{Analysis of the GSEMO and the \NSGA on the Bi-Objective Minimum Spanning Tree Problem} \label{sec:runtime}
In this section, we prove two main results on the expected runtime of the GSEMO and the \NSGA on the Bi-Objective Minimum Spanning Tree Problem.

To state our theorem on the \NSGA in the strongest possible form (which will enable us to easily obtain results for several different algorithm variants in Section~\ref{sec:applications}), we need to introduce two parameters which depend on the offspring generation mechanism which the \NSGA works with.

For any population $P$ containing no spanning tree, any $s \in F_1$ and any position $i \in \{1, \dots, n \}$ of a bit of value $0$, let $\mathbf{p}_1$ be a lower bound (that does not depend on $P, s, i$ and $j$) of the probability that there exists in the offspring a child generated with $s$ as (one of) the parent(s), that differs from $s$ on exactly bit $i$.

For any population $P$ containing at least a spanning tree, any $s \in F_1$, and any pair of bits of different value $i, j \in \{1, \dots, n \}$, let $\mathbf{p}_2$ be a lower bound (that does not depend on $P, s, i$ and $j$) of the probability that there exists in the offspring a child generated with $s$ as (one of) the parent(s), that differs from $s$ on exactly bits $i, j$.

\begin{theorem} \label{thm:main_nsga2}
The expected number of generations until the \NSGA, working on the fitness function $f$ with a population of size $N \ge 4((n-1)\wmin + 1)$ and with any offspring generation mechanism  resulting in $\mathbf{p}_1$ and $\mathbf{p}_2$, constructs a population which includes a spanning tree for each extremal vector of \convF is upper bounded by $O(\frac{\log n}{\mathbf{p}_1} + \frac{\log (n\wmax) }{\mathbf{p}_2}).$
\end{theorem} 

For the readers' convenience, we note already here that in typical versions of the \NSGA, we have $\mathbf{p}_1 = \Theta(1/m)$ and $\mathbf{p}_2 = \Theta(1/m^2)$, recall that the lenghth of the bit-string representation is $m$ for MST problems. Then the runtime bound above becomes $O(m^2 \log(n \wmax)$ generations. We refer to Section~\ref{sec:applications} for the details.

Our arguments developed in the proof of Theorem~\ref{thm:main_nsga2} easily yield the following runtime estimate for the GSEMO, which improves the previously known results asymptotically.

\begin{theorem}
The expected number of fitness evaluations until the GSEMO
working on the fitness function $f$ constructs
a population, which includes a spanning tree for
each extremal vector of $\convF$ is
upper bounded by $O(m^2n\wmin\log (n\wmax))$.\label{thm:main_gsemo}
\end{theorem}

The proof is split into three parts. First, we  prove that the \NSGA and the GSEMO find spanning trees in negligible time. Secondly, we introduce an elitism property of the \NSGA from the moment when the population includes spanning trees. Then, we bound the time taken by both \NSGA and GSEMO to find the extremal points from this point.

\subsection{Sampling the First Spanning Tree} 

Regarding the first phase of the optimization process, the proof of \cite[Lemma 5]{Neumann07} bounds the expected time until the population of the GSEMO contains at least one spanning tree from above, giving the following lemma.
\begin{lemma} \label{lem:tree_time} 
The GSEMO working on the fitness function $f$ constructs a population with at least one spanning tree in expected time $O(m \log n)$. 
\end{lemma}
We argue that a similar bound holds for the \NSGA.
\begin{lemma} \label{lem:tree_timeNSGA}
The \NSGA working on the fitness function $f$ constructs a population with at least one spanning tree in expected time $O(\frac{\log n}{\mathbf{p}_1})$. 
\end{lemma}
The proof is essentially an adaptation of \cite[Lemma 5]{Neumann07}, the details of which are in the appendix of the Arvix version of this paper. We first bound the time before the population contains a connected graph: to do so, we observe that for a subgraph $H$ with $l \ge 2$ connected components, there are at least $l-1$ edges of $G$ that decrease $l$. This allows us to use $\mathbf{p}_1$ to give a lower bound for the probability that these specific edges are added, and thus that connected components merge fast enough. A similar argument is used for the second phase, to show that excessive edges are deleted at roughly the same speed.

\subsection{An Elitism Property of the \NSGA}
One of the observations that make the study of the GSEMO easier is that it has an elitism property: an individual will not disappear from the population unless it is replaced by a dominating one. Here, we introduce a lemma, inspired by \cite{ZhengD22gecco}, which shows that under a certain condition on the population size, the \NSGA has a similar property.

\begin{lemma} \label{lem:pop_size}
Let $P$ be a population such that $|P| > 4((n-1)w_{\min} + 1) $ and having at least one spanning tree. Let $P'$ be the next population. For each individual $s$ in $P$, there is an individual $s'$ in $P'$ such that $s' \preceq s$.
\end{lemma}

The proposed proof for this lemma requires to introduce the notion of incomparable set, which is used to bound the size of $F_1$.

\begin{definition}[Incomparable set]
    $A\subseteq S$ is an incomparable set if there is no pair of individuals $s, s' \in A$, such that $s' \prec s$.
\end{definition}

We now make the following observation regarding the size of incomparable sets.

\begin{lemma}
\label{lem:inc_set_val_bound}
Let $P$ be an incomparable set containing at least one spanning tree. Then the number of objective values of $P$ is bounded by $(n-1)\wmin + 1$. 
\end{lemma}

\begin{proof}
Since $f$ is designed such that a spanning tree strictly dominates every non-spanningtree graph, $P$ consists of spanning trees only. Among two spanning trees of $P$ with the same $f_1$ value, one of them necessarily dominates the other. $P$ being an incomparable set, this does only happens if the two spanning trees have the same objective value. Thus, for a fixed value of $f_1$, there is at most one point in $f(P)$ having this $f_1$ value. Since a spanning tree has $n-1$ edges, and since by definition $0 \le w_1(e) \le w_{1}^{\max}$ for any edge $e$, we deduce that, for any spanning tree $s$, $0 \le f_1(s) \le (n-1)w_1^{\max}$.
It follows that $$|w(P)| \leq |f_1(ST)| \leq (n-1)w_{1}^{\max}+1,$$ where $ST$ is the set of elements of $S$ representing a spanning tree constructed from $G$.  We do the same with $f_2$, which gives $|w(P)|\le (n-1)w_{\min} + 1$, this bound being the minimum of the two bounds previously obtained.
\end{proof}

The subsequent formal proof for Lemma~\ref{lem:pop_size} is rather long, and may be found in appendix of the Arvix version of the paper. However, we give a sketch of the principal arguments: we observe that, by definition of the fronts, the lemma requires to be proven only for individuals of the first front. Then, we show that, for a given fitness value $p$, there are at most $4$ individuals that have fitness $p$ and positive crowding distance. Finally, we use Lemma~\ref{lem:inc_set_val_bound} to conclude that, because the \NSGA will first choose individuals with positive crowding distance, the size of the population is such that the new generation will represent all fitness values of the first front. 

\subsection{Sampling the Extremal Points of the Pareto Front}

This section gives upper bounds for the expected runtimes of the two algorithms, assuming the population contains at least a spanning tree, which completes the proof of Theorem~\ref{thm:main_nsga2} and~\ref{thm:main_gsemo}.

We build our analysis on the fact that the extremal points are the unique minimums for specific linear combinations of the weights $w_1$ and $w_2$.
The following lemma gives the combination for $q_i$, when $1 \le i \le r$. To state it more easily, we introduce the points $q_0 = q_1 + (0, 1)$ and $q_{r+1} = q_r + (1, 0)$. These points are chosen such that for $1 \le i \le r$, $w(S)$ lies entirely on one of the two half planes delimited by the lines $\overline{q_jq_{j+1}}$ and $\overline{q_{j-1}q_j}$

\begin{lemma}\label{lem:pointwiseDistance}
Let $i \in \{1,\ldots, r\}$. For all objective values~$p$, let 
\begin{equation*}
    \begin{split}
        d_i(p) = (w_1(q_{i-1}) - w_1(q_{i+1}))(w_2(p)-w_2(q_i)) \\ + (w_2(q_{i+1}) - w_2(q_{i-1}))(w_1(p)-w_1(q_i)).
    \end{split}
\end{equation*}
Then, for a given objective value $p = w(s)$ of some individual $s$, we have 
    $d_i(p) \geq 0$ and $d_i(p) = 0 \Leftrightarrow p=q_i$.
\end{lemma}

\begin{figure}[h]

\includegraphics[scale=0.28]{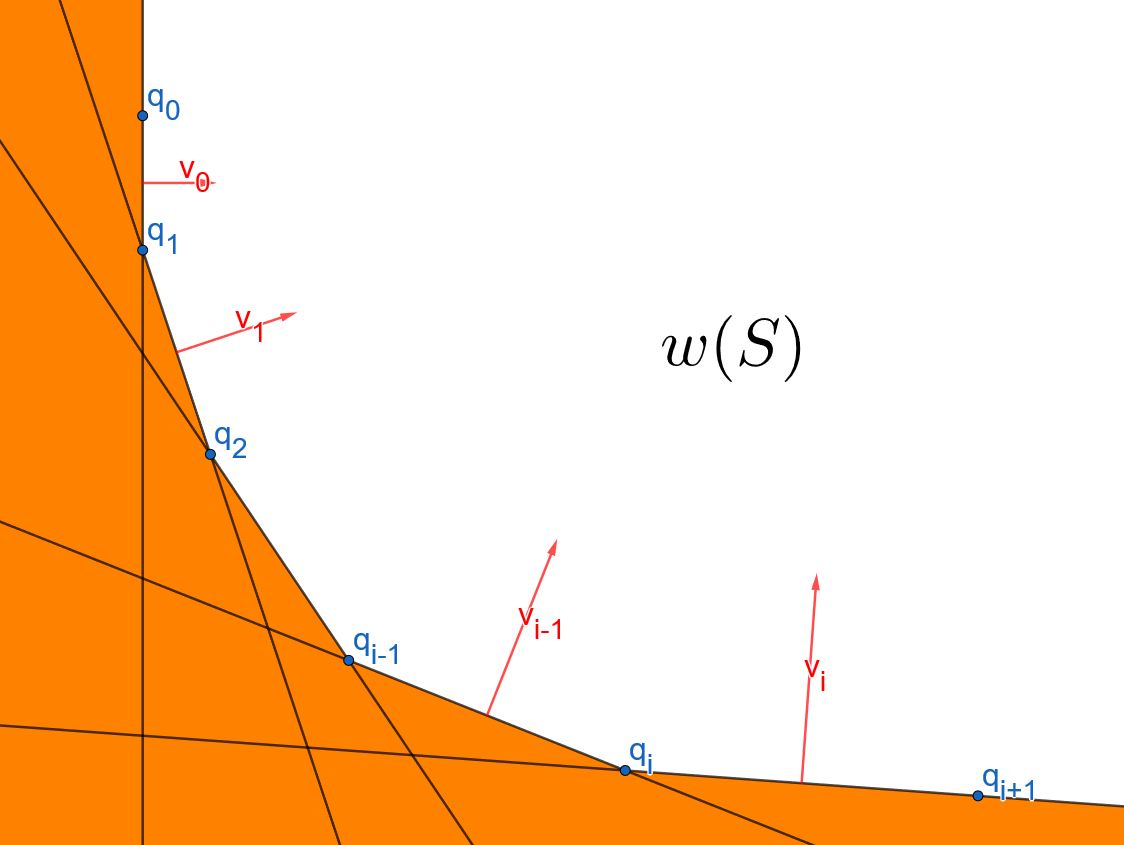}
\centering
\caption{Position of $w(S)$ relative to the lines $\overline{q_iq_{i+1}}$. Note that $q_0$ and $q_r$ may be on the Pareto front or out of $w(S)$, but that does not have any impact on our proof.}
\centering
\label{fig:intersection}
\end{figure}

\begin{proof}
    As shown in Figure~\ref{fig:intersection}, the set of all objective values is contained in the intersection of the upper-half planes $H_j$ of the lines $\overline{q_jq_{j+1}}, 0 \le j \le r$. For any such $j$, let 
\[  v_j = 
    \begin{pmatrix} 
    \frac{w_2(q_j) - w_2(q_{j+1})}{2} \\ \frac{w_1(q_{j+1}) - w_1(q_j)}{2}
    \end{pmatrix}.
\] 
    This vector is normal to the line $\overline{q_jq_{j+1}}$. Also, $w(S)$ lies entirely in one of the half planes delimited by $\overline{q_jq_{j+1}}$, that $v_j$ is pointing towards.
    Thus, for all $p \in \mathbb{R}^2, p \in H_j$ if and only if $ p \cdot v_j \ge q \cdot v_j$ for any point $q$ lying on $\overline{q_jq_{j+1}}$, with equality if and only if $p \in (q_j, q_{j+1})$.
    Now let us fix $i \in \{1,\ldots, r\}$. We know that $q_i$ lies on both $(q_i,q_{i+1})$ and $(q_{i-1},q_i)$, so, for all $p \in w(S), p \cdot v_i \ge q_i \cdot v_i$ and $p \cdot v_{i-1} \ge q_i \cdot v_{i-1}$, with equality if and only if $p \in (q_i,q_{i+1})$ or $p\in (q_{i-1},q_i)$, respectively.     
    By summing these two inequalities, we get $(p - q_i) \cdot (v_i + v_{i-1}) \ge 0$, that is $d_i(p) \ge 0$, with equality if and only if $p \in (q_i,q_{i+1}) \cap (q_i,q_{i-1})$. Since $q_i$ and $q_{i+1}$ are extremal points, these two lines cannot be parallel, hence $(q_i,q_{i+1}) \cap (q_iq_{i-1}) = \{q\}$, which concludes the proof.    
\end{proof}
We suppose from now on that the population $P$ contains a spanning tree, which is true after an expected number of $O(m \log n)$ iterations by Lemmas~\ref{lem:tree_time} and~\ref{lem:tree_timeNSGA}.
%

To prove the claimed upper bound, we use multiplicative drift analysis~\cite{DoerrJW12algo} on a quantity derived from these functions. A reminder for this technique may be found in appendix of the Arvix version of the paper.

To use the multiplicative drift theorem, we introduce a potential that has a multiplicative drift. 
For a given population~$P$, let 
\begin{align*}
d_i(P) &=  \min_{p \in w(P)} d_i(p),\\
d(P) &= \sum_{i=1}^r d_i(P).
\end{align*}

\begin{lemma}
\label{lem:d_noninc_nonneg}
Let $P(t)$ be the population of the \NSGA or the GSEMO algorithm after $t$ iterations. Then $d_i(P(t))$ for any $1 \le i \le r$, and $d(P(t))$ are nonnegative integers and  are non-increasing with respect to t.
\end{lemma}
\begin{proof}
Let $p_i(t) = \arg\min_{p \in w(P(t))} d_i(p)$.
Using Lemma~\ref{lem:pop_size} in the case of the \NSGA, and by the definition in the case of the GSEMO, there exists an element $p_i' \in w(P(t+1))$ such that $p_i' \preceq p_i$. 
Hence, \[\min_{p' \in w(P(t+1))} d_i(p')\leq d_i(p_i') \leq d_i(p_i) = \min_{p \in w(P(t))}d_i(p),\]
so $d_i(P(t))$ is non-increasing. Lemma~\ref{lem:pointwiseDistance} gives that $d_i(P(t))$ is nonnegative. Summing $d_i$ for $1 \le i \le r$ gives the same two properties for $d(P(t))$.
\end{proof}

In order to apply drift analysis, we need another lemma given by \cite{NeumannW07}.
\begin {lemma}[\cite{NeumannW07}, Lemma~2] \label{lem:neumann_weight_decrease}
Let $\Tilde{w} \colon S \longrightarrow \mathbb{R}$ be any function which is a linear combination with nonnegative coefficients of $w_1$ and $w_2$. Let $s$ be a search point describing a spanning tree $T$ . Let $\Tilde{w}_{\texttt{opt}}$ be the minimum value taken by $\Tilde{w}$ on spanning trees. Then there exists a set of $n$ 2-bit flips resulting on new spanning trees, such that the average weight decrease of these flips is at least $\frac{\Tilde{w}(s) - \Tilde{w}_{\texttt{opt}}}{n}$ 
\end{lemma}

We now prove the following lemma. 

\begin{lemma}
The expected number of generations until the \NSGA (resp. GSEMO), working on the fitness function $f$ with a population of size $N \ge 4((n-1)\wmin + 1)$ and with any offspring generation mechanism (resp. its intrinsic mechanism), constructs a population $P$ such that $d(P) = 0$ is upper bounded by $O(\frac{\log (n\wmax)}{\mathbf{p}_2})$ (resp. $O(m^2\wmin\log (n\wmax))$). 
\end{lemma} 
\begin{proof}

For any $i \in \{1, \dots, r\}$ and any iteration  $t$, let $s_{i,t}$ be an individual such that $d_i(s_{i,t}) = d_i(P(t))$. Then, let $\mathbf{p}_{2, \emph{(i,t)}}$ be a lower bound over all pairs of bit positions of the probability that there exists, in the offspring, a child generated with $s_{(i,t)}$ as (one of) the parent(s) that differs from $s_{(i, t)}$ on exactly this pair of bits.
For the \NSGA algorithm, by definition of $\mathbf{p}_2$, we have $\mathbf{p}_{2, \emph{(i,t)}} \ge \mathbf{p}_2$.

For the GSEMO algorithm, if there is at least one spanning tree in  the population, the size of the population is upper-bounded by $(n-1)\wmin + 1$ by Lemma~\ref{lem:inc_set_val_bound}. Therefore, 
\begin{equation*}
\resizebox{.95\linewidth}{!}{$
    \mathbf{p}_{2, (i,t)} \ge \dfrac{1}{(n-1)\wmin + 1} \cdot  \dfrac{(1-\frac{1}{m})^{m-2}}{m^2} \ge \dfrac{1}{((n-1)\wmin + 1)em^2}$}.
\end{equation*}
We also denote $\mathbf{p}_2 = \frac{1}{(n-1)\wmin em^2}$ when analysing GSEMO, such that we have $\mathbf{p}_{2, \emph{(i,t)}} \ge \mathbf{p}_2$ for both algorithms.

We now apply drift analysis on $d(P(t))$.
Indeed, for  $i \in \{1,\ldots,r\}$, $ d_i(p)$ is a linear combination of $w_1$ and $w_2$ with non-negative coefficients minimum $ d_i^{\texttt{opt}} = 0$.

Focusing only on the set of $n$ 2-bits flips given by Lemma~\ref{lem:neumann_weight_decrease} on a search point being one of $\arg \min_{s \in P} d_i(s)$, and noticing that, since $d_i$ is non-increasing (Lemma~\ref{lem:d_noninc_nonneg}), all other bit flips contribute positively to the drift, we have
\[\mathbb{E}[d_i(P(t))-d_i(P(t+1)) |  d_i(P(t)) = x] \geq \tfrac xn \cdot n \cdot \mathbf{p}_{2, \emph{(i,t)}}.  \] Using this estimate, we show the following lower bound for the drift of our potential. $$\mathbb{E}\left[ d(P(t)) - d(P(t+1)) \mid d(P(t)) = x\right] \ge x \cdot \mathbf{p}_2.$$ 

Let $T = \min\{ t \in \N \mid d(P(t))=0 \}$. Since $d(P)$ is an non-negative integer, the minimum strictly positive value of $d(P)$ is $1$. Since each edge has a weight at most $2w_{\max}^2$,  $d(P) \leq \sum_{i=1}^r 2m\wmax^2 \leq 2rn^2\wmax^2$.
Finally, since the convex hull is an incomparable set, we have $r \leq (n-1)w_{\min}+1$ by Lemma~\ref{lem:inc_set_val_bound}. 
Hence we obtain
\begin{align*}
\max_t{d(P(t))} &\leq ((n-1)w_{\min}+1)mw_{\max} \\
&= O(n^3w_{\max}w_{\min}).
\end{align*}
Using multiplicative drift theorem, we obtain
\begin{align*}
\mathbb{E} (T) & \leq \dfrac{1}{\mathbf{p}_2}\log{\dfrac{\max_{t \in \N}{d(P(t))}}{1}} \\
&= O\left(\dfrac{\log{n} + \log{w_{\max}}}{\mathbf{p}_2}\right).
\end{align*}

We then conclude the proof by plugging in the value of $\mathbf{p}_2$ for GSEMO, which does only one fitness evaluation per iteration.
\end{proof}

Combining the last lemma with Lemma~\ref{lem:pointwiseDistance}, and noticing that \NSGA and GSEMO do respectively $N$ and $1$ fitness evaluation(s) per generation, we conclude the proof of theorems~\ref{thm:main_nsga2} and~\ref{thm:main_gsemo}.

\section{Study of Offspring Generation Mechanisms of the \NSGA}
\label{sec:applications}
This section is dedicated to demonstrating applications of Theorem~\ref{thm:main_nsga2} on the study of the influence of particular offspring generation mechanisms on the expected runtime of the \NSGA before finding the extremal points.

We start with a very rudimentary mechanism, composed of fair selection (selecting every individual $s$ in $P$ as a parent),  and standard bitwise mutation with constant $c$, $0 \le c \le m $ (mutating every bit of a given parent with probability $\frac{c}{m}$). Then we can define the lower bounds on the success probabilities as follows.
\begin{align*}
\mathbf{p}_1 &= \frac{c}{m}\left(1-\frac{c}{m}\right)^{m-1} \ge \frac{c}{e^cm},\\ 
\mathbf{p}_2 &= \frac{c^2}{m^2}\left(1 - \frac{c}{m}\right)^{m-2} \ge \frac{c^2}{e^cm^2}. 
\end{align*}
Theorem~\ref{thm:main_nsga2} then gives the following corollary.

\begin{corollary}
The expected time until the \NSGA, working on the fitness function $s$ with a population of size $N \ge 4((n-1)\wmin + 1)$ and using standard bitwise mutation and fair selection, constructs a population which includes a spanning tree for each extremal vector of \convF is upper bounded by $O(m^2\log (n\wmax))$ generations, and 
\[
O(m^2N\log (n\wmax))
\]
fitness evaluations.
\end{corollary}

With note that the same result would hold when replacing bitwise mutation with fast mutation~\cite{DoerrLMN17}, which proved to be advantageous in some multi-objective problems~\cite{DoerrZ21aaai}.

Theorem~\ref{thm:main_nsga2} also ensures that adding crossover with constant probability $0 < q < 1$ does not worsen the asymptotic complexity of the \NSGA. To be more precise, let us consider a random monoparental reproduction scheme $\mathcal{M}$, which, from an individual $s \in S$, generates an offspring $\mathcal{M}(s)$. From this operator, and a given crossover operator $\mathcal{C}$ we derive a biparental reproduction scheme $\mathcal{M^\prime}$, which, from a pair of individuals $s, s^\prime$, generates, with probability $q$, the children $\mathcal{C}(s, s^\prime)$ and $\mathcal{C}(s^\prime, s)$, and with probability $(1-q)$, the children $\mathcal{M}(s)$, and  $\mathcal{M}(s^\prime)$. Now, consider an offspring generation scheme which puts the selected parents into pairs and applies $\mathcal{M}^\prime$ to these pairs. Finally, note that if $\mathbf{p}_1$, $\mathbf{p}_2$ and $\mathbf{p}_1^\prime$, $\mathbf{p}_2^\prime$ are the pairs of probabilities defined in~\ref{thm:main_nsga2} for the offspring generation mechanism using $\mathcal{M}$ or  $\mathcal{M}^\prime$, then we have trivially $\mathbf{p}_1^\prime \ge (1-q)\mathbf{p}_1$ and $\mathbf{p}_2^\prime \ge (1-q)\mathbf{p}_2$. This gives the following corollary.

\begin{corollary}
    For a given \NSGA implementation, replacing $\mathcal{M}$ by $\mathcal{M^\prime}$, as described above (i.e. adding crossover with constant probability) does not worsen the asymptotic expected runtime (up to a multiplicative constant).
\end{corollary}

Finally, to demonstrate the relevance of the $\mathbf{p}_2$ parameter, let us introduce a problem-specific ``balanced'' mutation operator, described in \cite{NeumannW07}. For any tree individual $s \in S$ the mutation operator flips each 1-bit with probability $\frac{1}{n-1}$, and each 0-bit with probability $\frac{1}{m - n + 1}$, and applies standard bitwise mutation otherwise. We use fair selection for this example. By definition, the value of $\mathbf{p}_1$ remains unchanged compared to standard bitwise mutation, that is, $O(\frac{1}{m})$. However, one can show that $\mathbf{p}_2$ is increased from $O(\frac{1}{m^2})$ to $O(\frac{1}{(m-n)n})$ (see appendix of the Arvix version of the paper). Theorem~\ref{thm:main_nsga2} automatically gives:

\begin{corollary}
    The expected runtime of the \NSGA, working on the fitness function $f$, with $|P| > 4((n-1)\wmin + 1) $, using fair selection and balanced mutation is $O((m-n)n\log (n\wmax) + m \log n)$ generations and 
    \[
    O(N((m - n)n\log (n\wmax) + m\log n))
    \]
    fitness evaluations.
\end{corollary}

\section{Conclusion}

In this first mathematical runtime analysis of the \NSGA on a combinatorial optimization problem, we provided a general approach to proving runtime guarantees for MOEAs solving the bi-objective MST problem. For the global SEMO, this gave a bounds lower than the previously known ones by a factor of $\Omega(|F|)$. More interestingly, we could prove the same performance guarantees for the much more complex \NSGA. Our result applies to several variants of the \NSGA, including some that use crossover. As for the simple global SEMO algorithm, we obtain better guarantees when employing a balanced mutation operator (which supports the general belief that analyses on simple toy algorithms can nevertheless give useful hint for the use of more complex algorithms). 

Overall, this work indicate that mathematical runtime analyses for the \NSGA are possible also for combinatorial optimization problems. In this first work, we mostly concentrated on proving performance guarantees at all. For future work, it would be interesting to derive more insights on how to optimally use the \NSGA
on the particular problem (we only saw that balanced mutation is preferable). That such results are possible in principle is again indicated by the simpler works on artificial benchmarks, where, e.g., \cite{DoerrQ23tec} gave some indications on the right mutation rate. Clearly, runtime analyses of the \NSGA on other combinatorial optimization problems would also be desirable to put this research direction on a broader basis.


\section*{Acknowledgments}
This work was supported by a public grant as part of the Investissements d'avenir project, reference ANR-11-LABX-0056-LMH, LabEx LMH and a fellowship via the International Exchange Program of \'Ecole Polytechnique.

{\small
\bibliographystyle{named}
\bibliography{ich_master,alles_ea_master,rest}

\begin{thebibliography}{}

\bibitem[\protect\citeauthoryear{Antipov and Doerr}{2021}]{AntipovD21algo}
Denis Antipov and Benjamin Doerr.
\newblock A tight runtime analysis for the ${(\mu+\lambda)}$~{EA}.
\newblock {\em Algorithmica}, 83:1054--1095, 2021.

\bibitem[\protect\citeauthoryear{Arroyo \bgroup \em et al.\egroup
  }{2008}]{ArroyoVV08}
Jos{\'{e}} Elias~Claudio Arroyo, Pedro~Sampaio Vieira, and Dalessandro~Soares
  Vianna.
\newblock A {GRASP} algorithm for the multi-criteria minimum spanning tree
  problem.
\newblock {\em Annals of Operations Research}, 159:125--133, 2008.

\bibitem[\protect\citeauthoryear{Auger and Doerr}{2011}]{AugerD11}
Anne Auger and Benjamin Doerr, editors.
\newblock {\em Theory of Randomized Search Heuristics}.
\newblock World Scientific Publishing, 2011.

\bibitem[\protect\citeauthoryear{B{\"{a}}ck}{1993}]{Back93}
Thomas B{\"{a}}ck.
\newblock Optimal mutation rates in genetic search.
\newblock In {\em International Conference on Genetic Algorithms, ICGA 1993},
  pages 2--8. Morgan Kaufmann, 1993.

\bibitem[\protect\citeauthoryear{Bian and Qian}{2022}]{BianQ22}
Chao Bian and Chao Qian.
\newblock Better running time of the non-dominated sorting genetic
  algorithm~{II} {(NSGA-II)} by using stochastic tournament selection.
\newblock In {\em Parallel Problem Solving From Nature, PPSN 2022}, pages
  428--441. Springer, 2022.

\bibitem[\protect\citeauthoryear{Bian \bgroup \em et al.\egroup
  }{2023}]{BianZLQ23}
Chao Bian, Yawen Zhou, Miqing Li, and Chao Qian.
\newblock Stochastic population update can provably be helpful in
  multi-objective evolutionary algorithms.
\newblock In {\em International Joint Conference on Artificial Intelligence,
  IJCAI 2023}. ijcai.org, 2023.
\newblock To appear. Available as arXiv:2306.02611.

\bibitem[\protect\citeauthoryear{Bossek and Grimme}{2017}]{BossekG17}
Jakob Bossek and Christian Grimme.
\newblock A {P}areto-beneficial sub-tree mutation for the multi-criteria
  minimum spanning tree problem.
\newblock In {\em Symposium Series on Computational Intelligence, {SSCI} 2017},
  pages 1--8. {IEEE}, 2017.

\bibitem[\protect\citeauthoryear{Brockhoff \bgroup \em et al.\egroup
  }{2008}]{BrockhoffFN08}
Dimo Brockhoff, Tobias Friedrich, and Frank Neumann.
\newblock Analyzing hypervolume indicator based algorithms.
\newblock In {\em Parallel Problem Solving from Nature, {PPSN} 2008}, pages
  651--660. Springer, 2008.

\bibitem[\protect\citeauthoryear{Crawford}{2019}]{Crawford19}
Victoria~G. Crawford.
\newblock An efficient evolutionary algorithm for minimum cost submodular
  cover.
\newblock In {\em International Joint Conference on Artificial Intelligence,
  {IJCAI} 2019}, pages 1227--1233. ijcai.org, 2019.

\bibitem[\protect\citeauthoryear{Dang \bgroup \em et al.\egroup
  }{2021}]{DangEL21aaai}
Duc{-}Cuong Dang, Anton~V. Eremeev, and Per~Kristian Lehre.
\newblock Escaping local optima with non-elitist evolutionary algorithms.
\newblock In {\em {AAAI} Conference on Artificial Intelligence, {AAAI} 2021},
  pages 12275--12283. {AAAI} Press, 2021.

\bibitem[\protect\citeauthoryear{Dang \bgroup \em et al.\egroup
  }{2023a}]{DangOSS23gecco}
Duc-Cuong Dang, Andre Opris, Bahare Salehi, and Dirk Sudholt.
\newblock Analysing the robustness of {NSGA-II} under noise.
\newblock In {\em Genetic and Evolutionary Computation Conference, GECCO 2023}.
  {ACM}, 2023.
\newblock To appear. Available as arXiv:2306.04525.

\bibitem[\protect\citeauthoryear{Dang \bgroup \em et al.\egroup
  }{2023b}]{DangOSS23aaai}
Duc-Cuong Dang, Andre Opris, Bahare Salehi, and Dirk Sudholt.
\newblock A proof that using crossover can guarantee exponential speed-ups in
  evolutionary multi-objective optimisation.
\newblock In {\em Conference on Artificial Intelligence, {AAAI} 2023}. {AAAI}
  Press, 2023.
\newblock To appear. Available as arXiv:2301.13687.

\bibitem[\protect\citeauthoryear{Deb \bgroup \em et al.\egroup
  }{2002}]{DebPAM02}
Kalyanmoy Deb, Amrit Pratap, Sameer Agarwal, and T.~Meyarivan.
\newblock A fast and elitist multiobjective genetic algorithm: {NSGA-II}.
\newblock {\em IEEE Transactions on Evolutionary Computation}, 6:182--197,
  2002.

\bibitem[\protect\citeauthoryear{Doerr and Neumann}{2020}]{DoerrN20}
Benjamin Doerr and Frank Neumann, editors.
\newblock {\em Theory of Evolutionary Computation---Recent Developments in
  Discrete Optimization}.
\newblock Springer, 2020.
\newblock Also available at
  \url{http://www.lix.polytechnique.fr/Labo/Benjamin.Doerr/doerr_neumann_book.html}.

\bibitem[\protect\citeauthoryear{Doerr and Qu}{2023a}]{DoerrQ23tec}
Benjamin Doerr and Zhongdi Qu.
\newblock A first runtime analysis of the {NSGA-II} on a multimodal problem.
\newblock {\em Transactions on Evolutionary Computation}, 2023.
\newblock \url{https://doi.org/10.1109/TEVC.2023.3250552}.

\bibitem[\protect\citeauthoryear{Doerr and Qu}{2023b}]{DoerrQ23LB}
Benjamin Doerr and Zhongdi Qu.
\newblock From understanding the population dynamics of the {NSGA-II} to the
  first proven lower bounds.
\newblock In {\em Conference on Artificial Intelligence, {AAAI} 2023}. {AAAI}
  Press, 2023.
\newblock To appear. Available as 2209.13974.

\bibitem[\protect\citeauthoryear{Doerr and Qu}{2023c}]{DoerrQ23crossover}
Benjamin Doerr and Zhongdi Qu.
\newblock Runtime analysis for the {NSGA-II}: Provable speed-ups from
  crossover.
\newblock In {\em Conference on Artificial Intelligence, {AAAI} 2023}. {AAAI}
  Press, 2023.
\newblock To appear. Available as arXiv:2208.08759.

\bibitem[\protect\citeauthoryear{Doerr and Zheng}{2021}]{DoerrZ21aaai}
Benjamin Doerr and Weijie Zheng.
\newblock Theoretical analyses of multi-objective evolutionary algorithms on
  multi-modal objectives.
\newblock In {\em Conference on Artificial Intelligence, {AAAI} 2021}, pages
  12293--12301. {AAAI} Press, 2021.

\bibitem[\protect\citeauthoryear{Doerr \bgroup \em et al.\egroup
  }{2012}]{DoerrJW12algo}
Benjamin Doerr, Daniel Johannsen, and Carola Winzen.
\newblock Multiplicative drift analysis.
\newblock {\em Algorithmica}, 64:673--697, 2012.

\bibitem[\protect\citeauthoryear{Doerr \bgroup \em et al.\egroup
  }{2015}]{DoerrDE15}
Benjamin Doerr, Carola Doerr, and Franziska Ebel.
\newblock From black-box complexity to designing new genetic algorithms.
\newblock {\em Theoretical Computer Science}, 567:87--104, 2015.

\bibitem[\protect\citeauthoryear{Doerr \bgroup \em et al.\egroup
  }{2016}]{DoerrGN16}
Benjamin Doerr, Wanru Gao, and Frank Neumann.
\newblock Runtime analysis of evolutionary diversity maximization for
  {OneMinMax}.
\newblock In {\em Genetic and Evolutionary Computation Conference, GECCO 2016},
  pages 557--564. {ACM}, 2016.

\bibitem[\protect\citeauthoryear{Doerr \bgroup \em et al.\egroup
  }{2017}]{DoerrLMN17}
Benjamin Doerr, Huu~Phuoc Le, R\'egis Makhmara, and Ta~Duy Nguyen.
\newblock Fast genetic algorithms.
\newblock In {\em Genetic and Evolutionary Computation Conference, GECCO 2017},
  pages 777--784. {ACM}, 2017.

\bibitem[\protect\citeauthoryear{Doerr \bgroup \em et al.\egroup
  }{2022}]{DoerrHP22}
Benjamin Doerr, Omar~El Hadri, and Adrien Pinard.
\newblock The $(1+(\lambda,\lambda))$ global {SEMO} algorithm.
\newblock In {\em Genetic and Evolutionary Computation Conference, GECCO 2022},
  pages 520--528. {ACM}, 2022.

\bibitem[\protect\citeauthoryear{Droste \bgroup \em et al.\egroup
  }{2002}]{DrosteJW02}
Stefan Droste, Thomas Jansen, and Ingo Wegener.
\newblock On the analysis of the (1+1) evolutionary algorithm.
\newblock {\em Theoretical Computer Science}, 276:51--81, 2002.

\bibitem[\protect\citeauthoryear{Ehrgott}{2005}]{Ehrgott05}
Matthias Ehrgott.
\newblock {\em Multicriteria Optimization}.
\newblock Springer, 2nd edition, 2005.

\bibitem[\protect\citeauthoryear{Friedrich \bgroup \em et al.\egroup
  }{2010}]{FriedrichHHNW10}
Tobias Friedrich, Jun He, Nils Hebbinghaus, Frank Neumann, and Carsten Witt.
\newblock Approximating covering problems by randomized search heuristics using
  multi-objective models.
\newblock {\em Evolutionary Computation}, 18:617--633, 2010.

\bibitem[\protect\citeauthoryear{Giel}{2003}]{Giel03}
Oliver Giel.
\newblock Expected runtimes of a simple multi-objective evolutionary algorithm.
\newblock In {\em Congress on Evolutionary Computation, {CEC} 2003}, pages
  1918--1925. {IEEE}, 2003.

\bibitem[\protect\citeauthoryear{Hamacher and Ruhe}{1994}]{HamacherR94}
Horst~W. Hamacher and G{\"{u}}nther Ruhe.
\newblock On spanning tree problems with multiple objectives.
\newblock {\em Annals of Operations Research}, 52:209--230, 1994.

\bibitem[\protect\citeauthoryear{Horoba}{2009}]{Horoba09}
Christian Horoba.
\newblock Analysis of a simple evolutionary algorithm for the multiobjective
  shortest path problem.
\newblock In {\em Foundations of Genetic Algorithms, {FOGA} 2009}, pages
  113--120. {ACM}, 2009.

\bibitem[\protect\citeauthoryear{Huang and Zhou}{2020}]{HuangZ20}
Zhengxin Huang and Yuren Zhou.
\newblock Runtime analysis of somatic contiguous hypermutation operators in
  {MOEA/D} framework.
\newblock In {\em Conference on Artificial Intelligence, {AAAI} 2020}, pages
  2359--2366. {AAAI Press}, 2020.

\bibitem[\protect\citeauthoryear{Huang \bgroup \em et al.\egroup
  }{2019}]{HuangZCH19}
Zhengxin Huang, Yuren Zhou, Zefeng Chen, and Xiaoyu He.
\newblock Running time analysis of {MOEA/D} with crossover on discrete
  optimization problem.
\newblock In {\em Conference on Artificial Intelligence, {AAAI} 2019}, pages
  2296--2303. {AAAI Press}, 2019.

\bibitem[\protect\citeauthoryear{Jansen \bgroup \em et al.\egroup
  }{2005}]{JansenJW05}
Thomas Jansen, Kenneth A.~De Jong, and Ingo Wegener.
\newblock On the choice of the offspring population size in evolutionary
  algorithms.
\newblock {\em Evolutionary Computation}, 13:413--440, 2005.

\bibitem[\protect\citeauthoryear{Jansen}{2013}]{Jansen13}
Thomas Jansen.
\newblock {\em Analyzing Evolutionary Algorithms -- The Computer Science
  Perspective}.
\newblock Springer, 2013.

\bibitem[\protect\citeauthoryear{Knowles and Corne}{2000}]{KnowlesC00}
Joshua~D. Knowles and David Corne.
\newblock Approximating the nondominated front using the {P}areto archived
  evolution strategy.
\newblock {\em Evolutionary Computation}, 8:149--172, 2000.

\bibitem[\protect\citeauthoryear{Knowles and Corne}{2001}]{KnowlesC01}
Joshua~D. Knowles and David~W. Corne.
\newblock A comparison of encodings and algorithms for multiobjective minimum
  spanning tree problems.
\newblock In {\em Congress on Evolutionary Computation, {CEC} 2001}, pages
  544--551. {IEEE}, 2001.

\bibitem[\protect\citeauthoryear{Lai and Zhou}{2020}]{LaiZ20}
Xinsheng Lai and Yuren Zhou.
\newblock Analysis of multiobjective evolutionary algorithms on the biobjective
  traveling salesman problem (1, 2).
\newblock {\em Multimedia Tools and Applications}, 79:30839--30860, 2020.

\bibitem[\protect\citeauthoryear{Laumanns \bgroup \em et al.\egroup
  }{2002}]{LaumannsTDZ02}
Marco Laumanns, Lothar Thiele, Kalyanmoy Deb, and Eckart Zitzler.
\newblock Combining convergence and diversity in evolutionary multiobjective
  optimization.
\newblock {\em Evolutionary Computation}, 10:263--282, 2002.

\bibitem[\protect\citeauthoryear{Li \bgroup \em et al.\egroup }{2016}]{LiZZZ16}
Yuan-Long Li, Yu-Ren Zhou, Zhi-Hui Zhan, and Jun Zhang.
\newblock A primary theoretical study on decomposition-based multiobjective
  evolutionary algorithms.
\newblock {\em IEEE Transactions on Evolutionary Computation}, 20:563--576,
  2016.

\bibitem[\protect\citeauthoryear{Majumder \bgroup \em et al.\egroup
  }{2020}]{MajumderKKP20}
Saibal Majumder, Mohuya~B. Kar, Samarjit Kar, and Tandra Pal.
\newblock Uncertain programming models for multi-objective shortest path
  problem with uncertain parameters.
\newblock {\em Soft Computing}, 24:8975--8996, 2020.

\bibitem[\protect\citeauthoryear{M{\"{u}}hlenbein}{1992}]{Muhlenbein92}
Heinz M{\"{u}}hlenbein.
\newblock How genetic algorithms really work: mutation and hillclimbing.
\newblock In {\em Parallel Problem Solving from Nature, PPSN 1992}, pages
  15--26. Elsevier, 1992.

\bibitem[\protect\citeauthoryear{Neumann and Theile}{2010}]{NeumannT10}
Frank Neumann and Madeleine Theile.
\newblock How crossover speeds up evolutionary algorithms for the
  multi-criteria all-pairs-shortest-path problem.
\newblock In {\em Parallel Problem Solving from Nature, PPSN 2010, Part {I}},
  pages 667--676. Springer, 2010.

\bibitem[\protect\citeauthoryear{Neumann and Wegener}{2006}]{NeumannW06emo}
Frank Neumann and Ingo Wegener.
\newblock Minimum spanning trees made easier via multi-objective optimization.
\newblock {\em Natural Computing}, 5:305--319, 2006.

\bibitem[\protect\citeauthoryear{Neumann and Wegener}{2007}]{NeumannW07}
Frank Neumann and Ingo Wegener.
\newblock Randomized local search, evolutionary algorithms, and the minimum
  spanning tree problem.
\newblock {\em Theoretical Computer Science}, 378:32--40, 2007.

\bibitem[\protect\citeauthoryear{Neumann and Witt}{2010}]{NeumannW10}
Frank Neumann and Carsten Witt.
\newblock {\em Bioinspired Computation in Combinatorial Optimization --
  Algorithms and Their Computational Complexity}.
\newblock Springer, 2010.

\bibitem[\protect\citeauthoryear{Neumann and Witt}{2022}]{NeumannW22}
Frank Neumann and Carsten Witt.
\newblock Runtime analysis of single- and multi-objective evolutionary
  algorithms for chance constrained optimization problems with normally
  distributed random variables.
\newblock In {\em International Joint Conference on Artificial Intelligence,
  {IJCAI} 2022}, pages 4800--4806. ijcai.org, 2022.

\bibitem[\protect\citeauthoryear{Neumann}{2007}]{Neumann07}
Frank Neumann.
\newblock Expected runtimes of a simple evolutionary algorithm for the
  multi-objective minimum spanning tree problem.
\newblock {\em European Journal of Operational Research}, 181:1620--1629, 2007.

\bibitem[\protect\citeauthoryear{Nguyen \bgroup \em et al.\egroup
  }{2015}]{NguyenSN15}
Anh~Quang Nguyen, Andrew~M. Sutton, and Frank Neumann.
\newblock Population size matters: rigorous runtime results for maximizing the
  hypervolume indicator.
\newblock {\em Theoretical Computer Science}, 561:24--36, 2015.

\bibitem[\protect\citeauthoryear{Parraga{-}Alava \bgroup \em et al.\egroup
  }{2017}]{ParragaDI17}
Jorge Parraga{-}Alava, M{\'{a}}rcio Dorn, and Mario Inostroza{-}Ponta.
\newblock Using local search strategies to improve the performance of {NSGA-II}
  for the multi-criteria minimum spanning tree problem.
\newblock In {\em Congress on Evolutionary Computation, {CEC} 2017}, pages
  1119--1126. {IEEE}, 2017.

\bibitem[\protect\citeauthoryear{Qian \bgroup \em et al.\egroup
  }{2019}]{QianYTYZ19}
Chao Qian, Yang Yu, Ke~Tang, Xin Yao, and Zhi{-}Hua Zhou.
\newblock Maximizing submodular or monotone approximately submodular functions
  by multi-objective evolutionary algorithms.
\newblock {\em Artificial Intelligence}, 275:279--294, 2019.

\bibitem[\protect\citeauthoryear{Rudolph}{1997}]{Rudolph97}
G{\"u}nter Rudolph.
\newblock {\em Convergence Properties of Evolutionary Algorithms}.
\newblock Verlag Dr.~Kov{\v a}c, 1997.

\bibitem[\protect\citeauthoryear{Thierens}{2003}]{Thierens03}
Dirk Thierens.
\newblock Convergence time analysis for the multi-objective counting ones
  problem.
\newblock In {\em Evolutionary Multi-Criterion Optimization, {EMO} 2003}, pages
  355--364. Springer, 2003.

\bibitem[\protect\citeauthoryear{Wietheger and Doerr}{2023}]{WiethegerD23}
Simon Wietheger and Benjamin Doerr.
\newblock A mathematical runtime analysis of the {N}on-dominated {S}orting
  {G}enetic {A}lgorithm {III} ({NSGA-III}).
\newblock In {\em International Joint Conference on Artificial Intelligence,
  {IJCAI} 2023}. ijcai.org, 2023.
\newblock To appear. Available as arXiv:2211.08202.

\bibitem[\protect\citeauthoryear{Witt}{2006}]{Witt06}
Carsten Witt.
\newblock Runtime analysis of the ($\mu$ + 1) {EA} on simple pseudo-{B}oolean
  functions.
\newblock {\em Evolutionary Computation}, 14:65--86, 2006.

\bibitem[\protect\citeauthoryear{Zheng and Doerr}{2022a}]{ZhengD22gecco}
Weijie Zheng and Benjamin Doerr.
\newblock Better approximation guarantees for the {NSGA-II} by using the
  current crowding distance.
\newblock In {\em Genetic and Evolutionary Computation Conference, GECCO 2022},
  pages 611--619. {ACM}, 2022.

\bibitem[\protect\citeauthoryear{Zheng and Doerr}{2022b}]{ZhengD22arxivmany}
Weijie Zheng and Benjamin Doerr.
\newblock Runtime analysis for the {NSGA-II:} proving, quantifying, and
  explaining the inefficiency for three or more objectives.
\newblock {\em CoRR}, abs/2211.13084, 2022.

\bibitem[\protect\citeauthoryear{Zheng \bgroup \em et al.\egroup
  }{2022}]{ZhengLD22}
Weijie Zheng, Yufei Liu, and Benjamin Doerr.
\newblock A first mathematical runtime analysis of the {N}on-{D}ominated
  {S}orting {G}enetic {A}lgorithm {II} ({NSGA-II}).
\newblock In {\em Conference on Artificial Intelligence, {AAAI} 2022}, pages
  10408--10416. {AAAI} Press, 2022.

\bibitem[\protect\citeauthoryear{Zhou \bgroup \em et al.\egroup
  }{2011}]{ZhouQLZSZ11}
Aimin Zhou, Bo-Yang Qu, Hui Li, Shi-Zheng Zhao, Ponnuthurai~Nagaratnam
  Suganthan, and Qingfu Zhang.
\newblock Multiobjective evolutionary algorithms: A survey of the state of the
  art.
\newblock {\em Swarm and Evolutionary Computation}, 1:32--49, 2011.

\bibitem[\protect\citeauthoryear{Zhou \bgroup \em et al.\egroup
  }{2019}]{ZhouYQ19}
Zhi-Hua Zhou, Yang Yu, and Chao Qian.
\newblock {\em Evolutionary Learning: Advances in Theories and Algorithms}.
\newblock Springer, 2019.

\end{thebibliography}
}

\clearpage
\appendix

\section{Appendix}

\subsection{Illustration of \convF}
\begin{figure}[H]
    \centering
    \includegraphics[scale = 0.75]{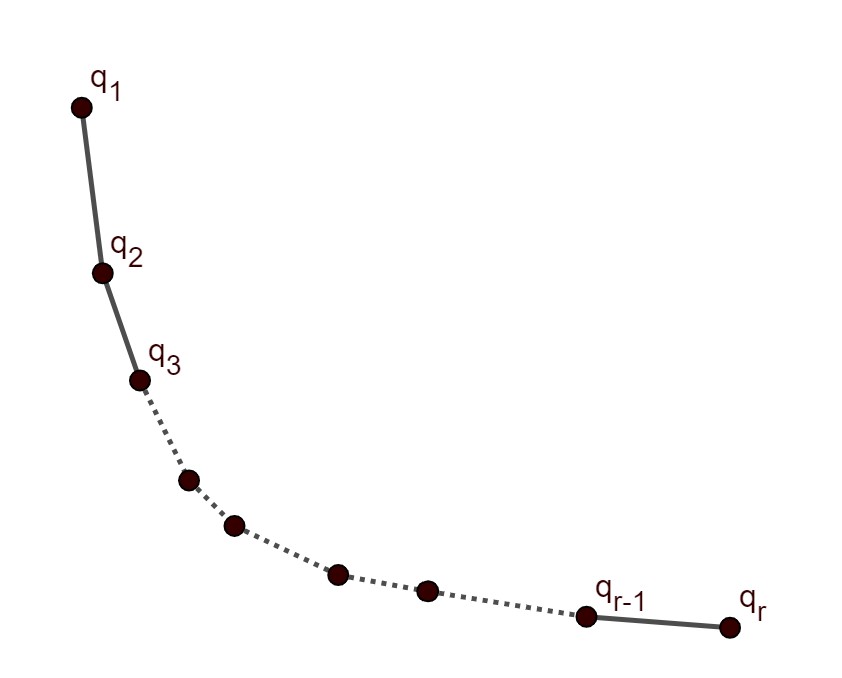}
    \caption{The lower-left part of the convex-hull of $F$. $w(S)$ lies at its top-right.}
    \vspace{20pt}
    \label{fig:convexhull}
\end{figure}
\vspace{10pt}

\subsection{Fast-Non-Dominated-Sort}

\begin{algorithm2e}[t]

\caption{fast-non-dominated-sort}\label{alg:fnds}
\Input{$X = \{x_1, \ldots, x_{|X|}\}$, the population}
\Output{$F_1, F_2, \ldots$ where $F_i$ is the set of individuals of rank $i$ }
\For{$i = 1$ \KwTo $|X|$}{
    $ND(x_i) = 0$ \tcp*{Number of open individuals dominating $x_i$}
    $XD(x_i) = \emptyset$  \tcp*{Individuals that $x_i$ dominates}
}
\For{$i = 1$ \KwTo $|X|$}{
    \For{$j = 1$ \KwTo $|X|$}{
        \If{$x_i \prec x_j$}{
            $ND(x_i) = ND(x_i) + 1$\\
            $XD(x_j) = XD(x_j) \cup \{x_i\}$
        }
    }
}
$F_1 = \{x_i \mid ND[i] = 0\}$\\ 
$k = 1$\\
\While{$F_k \neq \emptyset$}{
    $F_{k+1} = \emptyset$\\
    \For{$x \in F_k$}{
        \For{$x' \in XD(x)$}{
            $ND(x') = ND(x') - 1$\\
            \If{$ND(x') = 0$}{
                $F_{k+1} = F_{k+1} \cup {x'}$
            }
        }
    }
    $k = k+1$
}
\end{algorithm2e}

\subsection{Proof of Lemma~\ref{lem:tree_time}}

\begin{proof}
    First, we bound the expected time before the population contains a connected graph.
    Note that the fitness function is constructed in such a way that if $s, s' \in S$ are such that $c(s) < c(s')$, then $s \prec s'$. Thus, at any time, $F_1$ contains individuals all having the exact same number of components. Moreover, this number is non-increasing. Indeed, if $s$ is any individual in $F_1$ at iteration $t$, then any offspring having a bigger number of connected components than $s$ would be dominated by $s$, and thus would have rank at least $2$. Now, if at iteration $t$, $F_1$ consists of individuals with $l, 2 \le l \le n$ connected components, since $G$ is connected, there are for each individual at least $l - 1$ edges of $G$ whose inclusion reduces the number of components. Hence, by definition of $\mathbf{p}_1$ the probability that the number of connected components of individuals in $F_1$ decreases at iteration $t+1$ is lower bounded by $(l-1)\mathbf{p}_1$ (Note that the left term is the probability that at least one of the $l-1$ edges is added, but no other). Thus, the expected waiting time for a decrease is at most $\frac{1}{(l-1)\mathbf{p}_1}$. Hence, the expected time until the population consists only of connected graphs is upper bounded by $$\frac{1}{\mathbf{p}_1}\left(1 + \ldots + \frac{1}{n-1}\right) = O\left(\frac{\log n}{\mathbf{p}_1}\right).$$
    Now, we bound the expected time before the population contains a spanning tree, from the moment it contains a connected graph. From this point, we know that all individuals in $F_1$ are connected graphs. Also, still by construction of the fitness function, $F_1$ consists of individuals having the exact same number of edges, and we prove that this number is still non-increasing using the same argument as for connected components. If $F_1$ consists of non-spanningtree individuals with $N$ edges, $m \ge N \ge n$ , then for each individual in $F_1$, there are exactly $N-(n-1)$ edges whose exclusion decreases the number of edges without increasing the number of connected components. Using the same arguments as for connected graphs, we deduce that the expected time until $F_1$ contains spanning trees is at most $$O\left(\frac{\log(m-(n-1))}{\mathbf{p}_1}\right) = O\left(\frac{\log n}{\mathbf{p}_1}\right).$$ We get the claimed upper bound on the whole process simply by summing the two latter.\end{proof}

\subsection{Proof of Lemma~\ref{lem:pop_size}}
\begin{proof}
Let $S_{1.1}, . . . , S_{1.|F_1|}$ and $S_{2.1}, . . . , S_{2.|F_1|}$  be the population $F_1$ sorted  of $w_1$ and $w_2$ respectively. If $s \in F_i$, is at the position $k_1$ and $k_2$ in list $S_{1, .}$ and $S_{2, .}$ respectively, its crowding distance is positive if and only if one of the 4 following inequalities is true. 
\begin{itemize}
    \item $w_1(S_{1.k_1-1}) < w_1(S_{1.k_1})$
    \item $w_1(S_{1.k_1+1}) >  w_1(S_{1.k_1})$
    \item  $w_2(S_{2.k_2-1}) < w_2(S_{2.k_2})$
    \item$w_2(S_{2.k_2+1}) > w_2(S_{2.k_2}).$
\end{itemize}
However the lists $w_i(S_{i,.})$ are monotonic, meaning that for a fixed weight value $p_i \in \N$, the indices $k$ such that $w_i(S_{i, k}) = w$ form a contiguous segment of $\{0, \ldots, |F_1|\}$. Hence, only the two border indices of this segment satisfy one of the two conditions on $w_i$. We deduce that, for a fixed objective value $p \in \N^2$, there are at most $4$ individuals $s$ such that $w(s) = p$ having positive crowding distance.

Moreover, for a given objective value $(x, y) \in w(F_1)$,  there is at least one individual $s \in F_1$ such that $w(s) = (x, y)$ with a positive crowding distance. Indeed, since $F_1$ is incomparable, all individuals $s$ with $w_1(s)=x$ have an objective value $w_2(s) = y$, so if we denote $k_0$ the smallest index $k$ such that $w_1(S_{1, k}) = x$ in the sorted list, then $S_{1, k_0}$ is an individual satisfying these requirements. 
Let us now distinguish two cases. Now, let us distinguish two cases. 
\begin{itemize}
    \item If $|F_1| \le 4(n-1)\wmin + 1$, then, every individual in $F_1$ is selected for the next generation and so the lemma holds.
    \item if $|F_1| > 4(n-1)\wmin + 1$, first observe that $F_1$ consists only of spanning trees as it contains at least one spanning tree by assumption and a spanning tree dominates every non-spanning tree solution.
Moreover, $F_1$ is an incomparable set. Using Lemma~\ref{lem:inc_set_val_bound} it follows that $|w(F_1)| \le (n-1)w_{\min}$. 
But we know that for each $(x, y) \in w(F_1)$, there are at most 4 individuals $s$ with $cDist(s) > 0$ and $w(s) = (x, y)$. Thus, there are at most $4|w(F_1)| \leq 4((n-1)w_{\min} + 1) \leq |P|$ individuals in $F_1$ with positive crowding distance. By definition of the \NSGA algorithm, each individual from $F_1$ with positive crowding distance remains. Since there is at least one individual with positive crowding distance for each objective value, for any objective value $p \in w(F_1)$, there is an individual $s$ in the next iteration such that $w(s) = p$.
\end{itemize}

Now observe that for any $i > 1$, for any $s_i \in F_i$ by definition, there exists an $s_1 \in F_1$  such that $s_1 \preceq s_i$. So the individual in the next iteration which has the same objective value as $s_1$ will dominate $s_i$. This concludes the proof.
\end{proof}

\subsection{The Multiplicative Drift Theorem}
\begin{theorem}
[Multiplicative Drift Theorem \cite{DoerrN20}] \label{thm:mult_drift_analysis}
Let $(X_t)_{t \geq 0}$ be a sequence of non-negative random variables with a finite state space $S \subset \mathbb{R}^{+}$ such that $0 \in S$. Let $s_{min} = \min(S \backslash \{0\})$, let $T= \inf \{ t \geq 0 \mid X_t = 0 \}$ and, for $t \geq 0$ and $s \in S$, define  the drift $\Delta_t(s) = E[X_t - X_{t+1} \mid X_t = s]$. Suppose there exists $\delta > 0$ such that for all $s \in S \backslash \{0\}$ and all $t \geq 0$ the drift is 
$$\Delta_t(s) \geq \delta s.  $$
Then,
$$ E[T] \leq \frac{1+E[\log{(\frac{X_0}{s_{min}})}]}{\delta}.$$
\end{theorem}

\subsection{Lower bound for the drift of $d(P)$}

 We show the following lower bound for the drift of our potential. $$\mathbb{E}\left[ d(P(t)) - d(P(t+1)) \mid d(P(t)) = x\right] \ge x \cdot \mathbf{p}_2.$$
 \begin{proof}
Let $$q = \mathbf{P}[d_1(P(t)) = x_1,\ldots, d_r(P(t)) = x_r]$$ $$\Delta(t) = d(P(t)) - d(P(t+1))$$ $$\Delta_i(t) = d_i(P(t)) - d_i(P(t+1)).$$
Then, we have:

\begin{align*}
     \phantom{= }\mathbb{E}&\left[ \Delta(t) \mid d(P(t)) = x\right] = \mathbb{E}\left[ \mathbb{E}[\Delta(t) \mid  d(P(t)) = x\right] \mid d(P(t))=x]\\  
    = &\mathbb{E}\left[ \sum_{x_1+\ldots+x_r = x } (\sum_i \mathbb{E}[\Delta_i(t)  \mid d_i(P(t)) = x_i]) \cdot q \mid  d(P(t))=x \right] \\
    \geq &\mathbb{E}\left[ \sum_{x_1+\ldots+x_r = x } (\sum_i x_i \cdot \mathbf{p}_{2, (i,t)}) \cdot q\mid d(P(t))=x\right] \\
    \geq & \mathbb{E}\left[ \sum_{x_1+\ldots +x_r = x } x\cdot \mathbf{p}_2 \cdot q \mid d(P(t))=x\right] \\
    = &x \cdot \mathbf{p}_2.
\end{align*}

\subsection{$\mathbf{p}_2$-bound for balanced mutation and fair selection}

Let $s$ be any tree individual of the first front of any population $P$. Note that since $s$ is a tree, it has exactly $n-1$ edges, that is, $n-1$ 1-bits. Let $i, j \in \{1, \dots, m\}$ be the positions of a pair of bits of different values. Without loss of generality, we can suppose that the bit at position $i$ (respectively $j$) is a 1 (respectively a 0). Since we are using fair selection, $s$ will generate a child $s^\prime$ with probability $1$. Then, by definition of balanced mutation:
\begin{itemize}
    \item The probability that $s$ and $s^\prime$ differ on bit $i$ is $\frac{1}{n-1} = O(\frac{1}{n}).$
    \item The probability that $s$ and $s^\prime$ differ on bit $j$ is $\frac{1}{m-n+1} = O(\frac{1}{m-n}).$
    \item The probability that $s$ and $s^\prime$ are equal on all other bits is $(1 - \frac{1}{n-1})^{n-2}(1-\frac{1}{m-n+1})^{m-n} \ge \frac{1}{e^2} = O(1).$
\end{itemize}
By multiplying these three probabilities, we get the probability that $s$ and $s^\prime$ differ on exactly bits $i$ and $j$. Thus we can set $\mathbf{p}_2 = O(\frac{1}{n(m-n)})$
     
 \end{proof}

\end{document}


\appendix
\section{Appendix to the IJCAI 2023 Main Track Submission 5311}

\subsection{Illustration of \convF}
\begin{figure}[H]
    \centering
    \includegraphics[scale = 0.75]{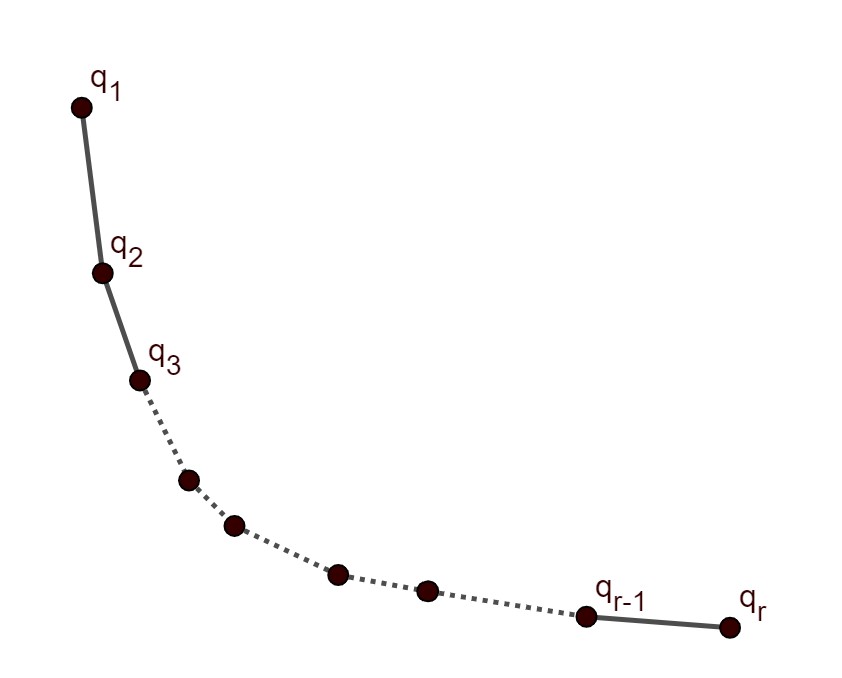}
    \caption{The lower-left part of the convex-hull of $F$. $w(S)$ lies at its top-right.}
    \vspace{20pt}
    \label{fig:convexhull}
\end{figure}
\vspace{10pt}

\subsection{Fast-Non-Dominated-Sort}

\begin{algorithm2e}[H]

\caption{fast-non-dominated-sort}\label{alg:fnds}
Input: $X = \{x_1, \ldots, x_{|X|}\}$, the population\\
Output: $F_1, F_2, \ldots$, where $F_i$ is the set of individuals of rank $i$ \\
\For{$i = 1$ \KwTo $|X|$}{
    $ND(x_i) = 0$ \tcp*{Number of open individuals dominating $x_i$}
    $XD(x_i) = \emptyset$  \tcp*{Individuals that $x_i$ dominates}
}
\For{$i = 1$ \KwTo $|X|$}{
    \For{$j = 1$ \KwTo $|X|$}{
        \If{$x_i \prec x_j$}{
            $ND(x_i) = ND(x_i) + 1$\\
            $XD(x_j) = XD(x_j) \cup \{x_i\}$
        }
    }
}
$F_1 = \{x_i \mid ND[i] = 0\}$\\ 
$k = 1$\\
\While{$F_k \neq \emptyset$}{
    $F_{k+1} = \emptyset$\\
    \For{$x \in F_k$}{
        \For{$x' \in XD(x)$}{
            $ND(x') = ND(x') - 1$\\
            \If{$ND(x') = 0$}{
                $F_{k+1} = F_{k+1} \cup {x'}$
            }
        }
    }
    $k = k+1$
}
\end{algorithm2e}

\subsection{Proof of Lemma 9}

\begin{proof}
    First, we bound the expected time before the population contains a connected graph.
    Note that the fitness function is constructed in such a way that if $s, s' \in S$ are such that $c(s) < c(s')$, then $s \prec s'$. Thus, at any time, $F_1$ contains individuals all having the exact same number of components. Moreover, this number is non-increasing over time. Indeed, if $s$ is any individual in $F_1$ at iteration $t$, then any offspring having a bigger number of connected components than $s$ would be dominated by $s$, and thus would have rank at least $2$. 
    
    Assume that in some iteration $t$, the first front $F_1$ consists of individuals with $l, 2 \le l \le n$ connected components. Since $G$ is connected, for each individual there are at least $l - 1$ edges of $G$ whose inclusion reduces the number of components. Hence, by definition of $\mathbf{p}_1$ the probability that the number of connected components of individuals in $F_1$ decreases at iteration $t+1$ is at least $(l-1)\mathbf{p}_1$.
    Thus, the expected waiting time for a decrease is at most $\frac{1}{(l-1)\mathbf{p}_1}$. Hence, the expected time until the population consists only of connected graphs is upper bounded by $$\frac{1}{\mathbf{p}_1}\left(1 + \ldots + \frac{1}{n-1}\right) = O\left(\frac{\log n}{\mathbf{p}_1}\right).$$
    
    Now, we bound the expected time until the population contains a spanning tree, from the moment it contains a connected graph. From this point, we know that all individuals in $F_1$ are connected graphs. Again, by construction of the fitness function, $F_1$ consists of connected graph individuals having the exact same number of edges, and we prove that this number is non-increasing using the same argument as for connected components. If $F_1$ consists of individuals with $N$ edges, $m \ge N \ge n$ , then for each individual in $F_1$, there are exactly $N-(n-1)$ edges whose exclusion decreases the number of edges without increasing the number of connected components. Using the same arguments as for connected graphs, we deduce that the expected time until $F_1$ contains spanning trees is at most $$O\left(\frac{\log(m-(n-1))}{\mathbf{p}_1}\right) = O\left(\frac{\log n}{\mathbf{p}_1}\right).$$ 
    
    We obtain the claimed upper bound on the whole process simply by summing the two times just computed.\end{proof}

\subsection{Proof of Lemma 10}
\begin{proof}
Let $S_{1.1}, . . . , S_{1.|F_1|}$ and $S_{2.1}, . . . , S_{2.|F_1|}$  be the population $F_1$ sorted according $w_1$ and $w_2$ respectively. If $s \in F_1$ is at the position $k_1$ and $k_2$ in list $S_{1, .}$ and $S_{2, .}$ respectively, its crowding distance is positive if and only if one of the four following inequalities is true: 
\begin{itemize}
    \item $w_1(S_{1.k_1-1}) < w_1(S_{1.k_1})$,
    \item $w_1(S_{1.k_1+1}) >  w_1(S_{1.k_1})$,
    \item  $w_2(S_{2.k_2-1}) < w_2(S_{2.k_2})$,
    \item$w_2(S_{2.k_2+1}) > w_2(S_{2.k_2}).$
\end{itemize}
However the lists $w_i(S_{i,.})$ are monotonic, meaning that for a fixed weight value $p_i \in \N$, the indices $k$ such that $w_i(S_{i, k}) = p_i$ form a contiguous segment of $\{0, \ldots, |F_1|\}$. Hence, only the two border indices of this segment satisfy one of the two conditions on $w_i$. We deduce that, for a fixed objective value $p \in \N^2$, there are at most $4$ individuals $s$ with $w(s) = p$ and a positive crowding distance.

Moreover, for a given objective value $(x, y) \in w(F_1)$,  there is at least one individual $s \in F_1$ such that $w(s) = (x, y)$ with a positive crowding distance. Indeed, since $F_1$ is incomparable, all individuals $s$ with $w_1(s)=x$ have an objective value $w_2(s) = y$, so if we denote $k_0$ the smallest index $k$ such that $w_1(S_{1, k}) = x$ in the sorted list, then $S_{1, k_0}$ is an individual satisfying these requirements. 
Let us now distinguish two cases. Now, let us distinguish two cases: 
\begin{itemize}
    \item If $|F_1| \le 4(n-1)\wmin + 1$, then, every individual in $F_1$ is selected for the next generation and so the lemma holds.
    \item if $|F_1| > 4(n-1)\wmin + 1$, first observe that $F_1$ consists only of spanning trees as it contains at least one spanning tree by assumption and a spanning tree dominates every non-spanning tree solution.
Moreover, $F_1$ is an incomparable set. Using Lemma 12 it follows that $|w(F_1)| \le (n-1)w_{\min}$. 
But we know that for each $(x, y) \in w(F_1)$, there are at most 4 individuals $s$ with $cDist(s) > 0$ and $w(s) = (x, y)$. Thus, there are at most $4|w(F_1)| \leq 4((n-1)w_{\min} + 1) \leq |P|$ individuals in $F_1$ with positive crowding distance. By definition of the \NSGA algorithm, each individual from $F_1$ with positive crowding distance remains. Since there is at least one individual with positive crowding distance for each objective value, for any objective value $p \in w(F_1)$, there is an individual $s$ in the next iteration such that $w(s) = p$.
\end{itemize}

Now observe that for any $i > 1$, for any $s_i \in F_i$ by definition, there exists an $s_1 \in F_1$  such that $s_1 \preceq s_i$. So the individual in the next iteration which has the same objective value as $s_1$ will dominate $s_i$. This concludes the proof.
\end{proof}

\subsection{The Multiplicative Drift Theorem}
\begin{theorem}
[Multiplicative Drift Theorem \cite{DoerrJW12algo}] \label{thm:mult_drift_analysis}
Let $(X_t)_{t \geq 0}$ be a sequence of non-negative random variables with a finite state space $S \subset \mathbb{R}^{+}$ such that $0 \in S$. Let $s_{min} = \min(S \backslash \{0\})$, let $T= \inf \{ t \geq 0 \mid X_t = 0 \}$ and, for $t \geq 0$ and $s \in S$, define  the drift $\Delta_t(s) = E[X_t - X_{t+1} \mid X_t = s]$. Suppose there exists $\delta > 0$ such that for all $s \in S \backslash \{0\}$ and all $t \geq 0$ the drift is 
$$\Delta_t(s) \geq \delta s.  $$
Then
$$ E[T] \leq \frac{1+E[\log{(\frac{X_0}{s_{min}})}]}{\delta}.$$
\end{theorem}

\subsection{Lower bound for the drift of $d(P)$}

 We show the following lower bound for the drift of our potential: $$\mathbb{E}\left[ d(P(t)) - d(P(t+1)) \mid d(P(t)) = x\right] \ge x \cdot \mathbf{p}_2.$$
 \begin{proof}
Let $$q = \mathbf{P}[d_1(P(t)) = x_1,\ldots, d_r(P(t)) = x_r]$$ $$\Delta(t) = d(P(t)) - d(P(t+1))$$ $$\Delta_i(t) = d_i(P(t)) - d_i(P(t+1)).$$
Then, summing over all $x_1, \dots, x_r$ such that $x_1+\ldots+x_r = x$, we compute
\begin{align*}
     \mathbb{E}&\left[ \Delta(t) \mid d(P(t)) = x\right] \\
     =& \mathbb{E}\left[ \mathbb{E}[\Delta(t) \mid  d(P(t)) = x\right] \mid d(P(t))=x]\\  
    = &\mathbb{E}\left[ \sum_{x_i} \left( \sum_i \mathbb{E}[\Delta_i(t)  \mid d_i(P(t)) = x_i]\right) \cdot q \bigm\vert  d(P(t))=x \right] \\
    \geq &\mathbb{E}\left[ \sum_{x_i} \left(\sum_i x_i \cdot \mathbf{p}_{2, (i,t)}\right) \cdot q \bigm\vert d(P(t))=x\right] \\
    \geq & \mathbb{E}\left[ \sum_{x_i} x\cdot \mathbf{p}_2 \cdot q \bigm\vert d(P(t))=x\right] \\
    = &x \cdot \mathbf{p}_2.
\end{align*}
\end{proof}

\subsection{$\mathbf{p}_2$-bound for Balanced Mutation and Fair Selection}

Let $s$ be any tree individual of the first front of any population $P$. Note that since $s$ is a tree, it has exactly $n-1$ edges, that is, $n-1$ 1-bits. Let $i, j \in \{1, \dots m\}$ be the positions of a pair of bits of different values. Without loss of generality, we can suppose that the bit at position $i$ (respectively $j$) is a 1 (respectively a 0). Since we are using fair selection, $s$ will generate a child $s^\prime$ with probability $1$. Then, by definition of the balanced mutation operator:
\begin{itemize}
    \item The probability that $s$ and $s^\prime$ differ on bit $i$ is $\frac{1}{n-1} = O(\frac{1}{n}).$
    \item The probability that $s$ and $s^\prime$ differ on bit $j$ is $\frac{1}{m-n+1} = O(\frac{1}{m-n}).$
    \item The probability that $s$ and $s^\prime$ are equal on all other bits is $(1 - \frac{1}{n-1})^{n-2}(1-\frac{1}{m-n+1})^{m-n} \ge \frac{1}{e^2} = O(1).$
\end{itemize}
By multiplying these three probabilities, we get the probability that $s$ and $s^\prime$ differ on exactly bits $i$ and $j$. Thus we can set $\mathbf{p}_2 = O(\frac{1}{n(m-n)})$.
     
 \bibliographystyle{named}
\bibliography{ich_master,alles_ea_master,rest}